\ifdefined\isarxiv
\documentclass[11pt]{article}
\usepackage[numbers]{natbib}

\usepackage{subfig}
\else

\documentclass[conference]{IEEEtran}
\IEEEoverridecommandlockouts

\def\BibTeX{{\rm B\kern-.05em{\sc i\kern-.025em b}\kern-.08em
    T\kern-.1667em\lower.7ex\hbox{E}\kern-.125emX}}

\usepackage{subfig}
\fi

\usepackage{amsmath}
\usepackage{amsthm} %
\usepackage{amssymb}
\usepackage{algorithm}
\usepackage{algpseudocode}
\usepackage{graphicx}
\usepackage{grffile}
\usepackage{url}
\usepackage{xcolor}
\usepackage{epstopdf}

\usepackage{bbm}
\usepackage{dsfont}

\allowdisplaybreaks

\ifdefined\isarxiv

\usepackage{tikz}
\usepackage{hyperref}  
\hypersetup{colorlinks=true,citecolor=blue,linkcolor=blue} 
\usetikzlibrary{arrows}
\usepackage[margin=1in]{geometry}

\else

\usepackage{microtype}
\usepackage{hyperref}
\definecolor{mydarkblue}{rgb}{0,0.08,0.45}
\hypersetup{colorlinks=true, citecolor=mydarkblue,linkcolor=mydarkblue}

\fi
\graphicspath{{./figs/}}

\newtheorem{theorem}{Theorem}[section]
\newtheorem{lemma}[theorem]{Lemma}
\newtheorem{definition}[theorem]{Definition}

\newcommand{\wh}{\widehat}
\newcommand{\wt}{\widetilde}

\newcommand{\R}{\mathbb{R}}

\renewcommand{\tilde}{\wt}
\renewcommand{\hat}{\wh}
\newcommand{\Tmat}{{\cal T}_{\mathrm{mat}}}

\DeclareMathOperator*{\E}{{\mathbb{E}}}

\DeclareMathOperator{\poly}{poly}

\DeclareMathOperator{\tr}{tr}

\makeatletter
\newcommand*{\RN}[1]{\expandafter\@slowromancap\romannumeral #1@}
\makeatother

\usepackage{lineno}

\makeatletter 
\newcommand{\linebreakand}{%
  \end{@IEEEauthorhalign}
  \hfill\mbox{}\par
  \mbox{}\hfill\begin{@IEEEauthorhalign}
}
\makeatother

\begin{document}

\date{}
\title{Fast Heavy Inner Product Identification Between Weights and Inputs in Neural Network Training}

\ifdefined\isarxiv


\author{
Lianke Qin\thanks{\texttt{lianke@ucsb.edu}. UCSB.}
\and 
Saayan Mitra\thanks{\texttt{smitra@adobe.com}. Adobe Research.}
\and 
Zhao Song\thanks{\texttt{zsong@adobe.com}. Adobe Research.}
\and
Yuanyuan Yang\thanks{\texttt{yyangh@cs.washington.edu}. University of Washington.}
\and 
Tianyi Zhou\thanks{\texttt{tzhou29@usc.edu}. USC.}
}

\else

\author{
\IEEEauthorblockN{Lianke Qin}
\IEEEauthorblockA{\textit{Department of Computer Science} \\
\textit{University of California, Santa Barbara}\\
Santa Barbara, CA \\
lianke@ucsb.edu}
\and
\IEEEauthorblockN{Saayan Mitra}
\IEEEauthorblockA{\textit{Adobe Research} \\
\textit{Adobe}\\
San Jose, CA \\
smitra@adobe.com}
\and
\IEEEauthorblockN{Zhao Song}
\IEEEauthorblockA{
\textit{Adobe Research}\\
\textit{Adobe}\\
San Jose, CA \\
zsong@adobe.com}
\linebreakand 
\IEEEauthorblockN{Yuanyuan Yang}
\IEEEauthorblockA{
\textit{Department of Computer Science}\\
\textit{University of Washington}\\
Seattle, WA \\
yyangh@cs.washington.edu}
\and 
\IEEEauthorblockN{Tianyi Zhou}
\IEEEauthorblockA{
\textit{Department of Computer Science}\\
\textit{University of Southern California}\\
Los Angeles, CA \\
tzhou029@usc.edu}
}

\fi

\ifdefined\isarxiv
\begin{titlepage}
  \maketitle
  \begin{abstract}
In this paper, we consider a heavy inner product identification problem, which generalizes the Light Bulb problem~(\cite{prr89}): Given two sets $A \subset \{-1,+1\}^d$ and $B \subset \{-1,+1\}^d$ with $|A|=|B| = n$, if there are exact $k$ pairs whose inner product passes a certain threshold, i.e., $\{(a_1, b_1), \cdots, (a_k, b_k)\} \subset A \times B$ such that $\forall i \in [k], \langle a_i,b_i \rangle \geq \rho \cdot d$, for a threshold $\rho \in (0,1)$, the goal is to identify those $k$ heavy inner products. We provide an algorithm that runs in 
$O(n^{2 \omega / 3+ o(1)})$ time to find the $k$ inner product pairs that surpass $\rho \cdot d$ threshold with high probability, where $\omega$ is the current matrix multiplication exponent. By solving this problem, our method speed up the training of neural networks with ReLU activation function.

  \end{abstract}
  \thispagestyle{empty}
\end{titlepage}

{\hypersetup{linkcolor=black}
}
\newpage

\else
\maketitle

\begin{abstract}

\end{abstract}

\fi

\section{Introduction}



Efficiently finding heavy inner product has been shown useful in many fundamental optimization tasks such as sparsification \cite{sxz22}, linear program with small tree width \cite{y20,gs22}, frank-wolf~\cite{xss21}, reinforcement learning~\cite{ssx21}. In this paper, we define and study the heavy inner product identification problem, and provide both a randomized and a deterministic algorithmic solution for this problem. 

Mathematically, we define our main problem as follows:
\begin{definition}[Heavy inner product between two sets]\label{def:main_problem}
Given two sets $A \subset \{-1,+1\}^d$ and $B \subset \{-1,+1\}^d$ with $|A|=|B|=n$, 
which are all independently and uniformly random except for $k$ vector pairs $\{(a_1, b_1),\cdots, (a_k, b_k) \} \subset A \times B$ such that $\forall i \in [k], \langle a_i,b_i\rangle \geq \rho \cdot d$ for some $0 < \rho \leq 1$. The goal is to find these $k$ correlated pairs.
\end{definition}

We give an example of this problem in Figure \ref{fig:distribution}. A naive solution to this problem would be, to do a linear scan on every $(a,b)$ pair, which runs in $O(n^2 d)$ time. In practice, usually $k$ is sufficiently small, so that a $o(n^2 d)$ algorithm might be possible. This motivates the following question:

\begin{center}
    {\it Can we design an efficient algorithm to identify the heavy inner product pair $(a,b) \in A \times B$, i.e. runs in $o(n^2 d)$ time.}
\end{center}

We provide a positive answer (Theorem \ref{thm:main_prelim}) for the above question by an algorithmic solution \ifdefined\isshort (Algorithm 2 in full version \cite{full}) \else (Algorithm \ref{alg:light_bulb2}) \fi
. 
In the algorithm, we use two pairwise independent hash functions to partition $A$ and $B$ into $h = n^{2/3}$ groups of size $n^{1/3}$ as $A_1, \cdots, A_h$ and $B_1, \cdots, B_h$, respectively. After that, we compute a score $C_{i, j}$ for each group pair $A_i$ and $B_j$ to help identify whether the group pair contains the correlated vector pair with constant success probability. After repeating the process for $O(\log n)$ times, we can locate the target group pairs with polynomially low error 

and brute force search within group pairs to find the exact heavy inner product vector pairs. We further accelerate the computation of $C_{i,j}$ by carefully designing matrix multiplication.

\subsection{Our results}
We state our main results in the following theorem:
\begin{theorem}[Our results, informal version of Theorem~\ref{thm:main}]\label{thm:main_prelim}  
For the heavy inner product identification problem defined in Def.~\ref{def:main_problem},  
there is a constant $c_0>0$ such that for correlation $\rho$, we can find the heavy $k$ pairs $\{(a_1, b_1),\cdots, (a_k, b_k) \}$ in randomized time $O(n^{2 \omega / 3+ o(1)})$ whenever $d \leq c_0 \log n$ with
probability at least $1 -  1/n^{10}- k^2/n^{2/3}$. 
\end{theorem}
For the current fast matrix multiplication algorithm with $\omega \approx 2.373$~(\cite{aw21}), our running time becomes $O(n^{1.582})$.

\subsection{Technique Overview}

We briefly present our techniques in designing the algorithm: (1)\ Setting a high probability threshold for the uncorrelated inner product pairs. (2)\ Partition $A$ and $B$ into $h = n^{2/3}$ groups respectively and locate the group pair $(i,j)$ with the heavy inner product pair by a score $C_{i,j}$  
(3)\ Accelerate the computation of score function $C_{i,j}$.

\paragraph{Set the high probability threshold.}
By picking a threshold $v:=\delta(w / \kappa) \log n$, for large enough $\delta$, constant $w$ and for fixed $\kappa$, 
we have that for each uncorrelated pair $(x, y) \in A \times B$, $|\langle x, y \rangle| < v$ with probability at least $1 - 1/n^{13}$. By a union bound over all possible $O(n^2)$ pairs, we have $|\langle x, y\rangle| < v$ for all such $x, y$ with probability at least $1-1 / n^{11}$.
Let $\wt{x} \in A, \wt{y} \in B$ denote the correlated pairs which satisfy that $\langle \wt{x}, \wt{y}\rangle \geq \rho d= v$.

\begin{figure}[!ht]
    \centering
    \subfloat[]{\includegraphics[width=0.24\textwidth]{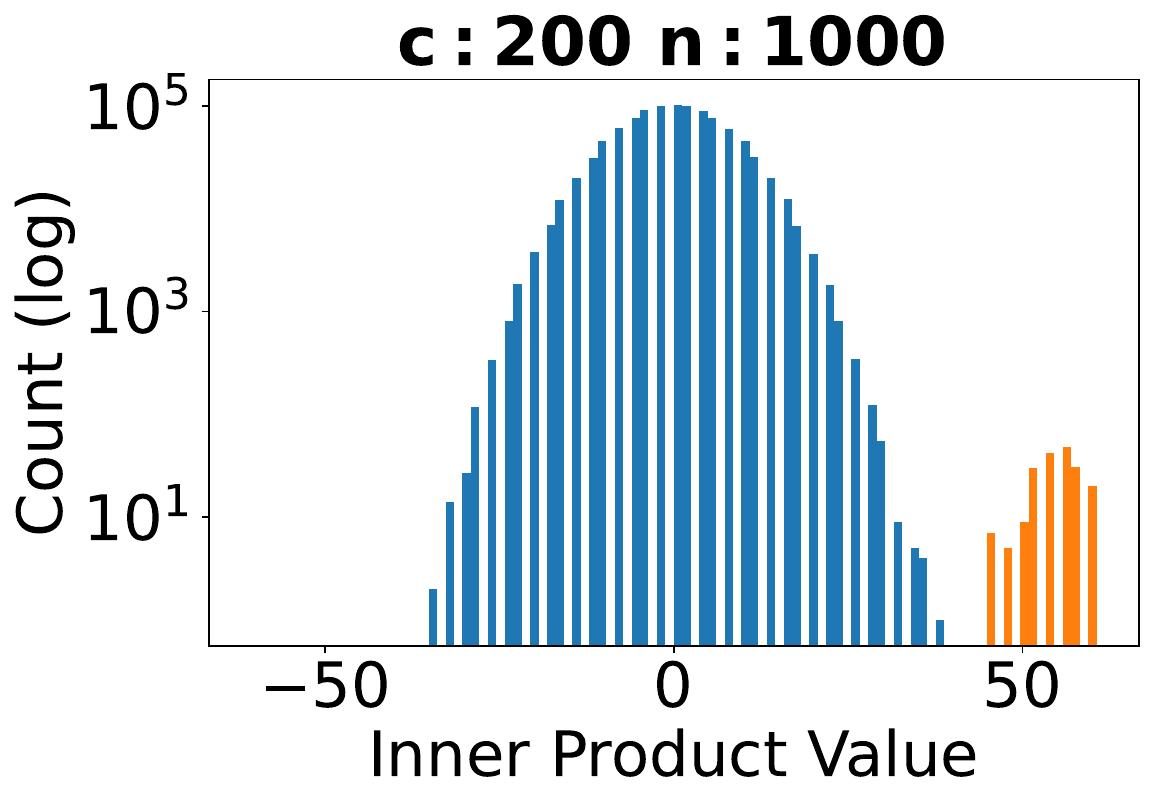}}
    \subfloat[]{\includegraphics[width=0.24\textwidth]{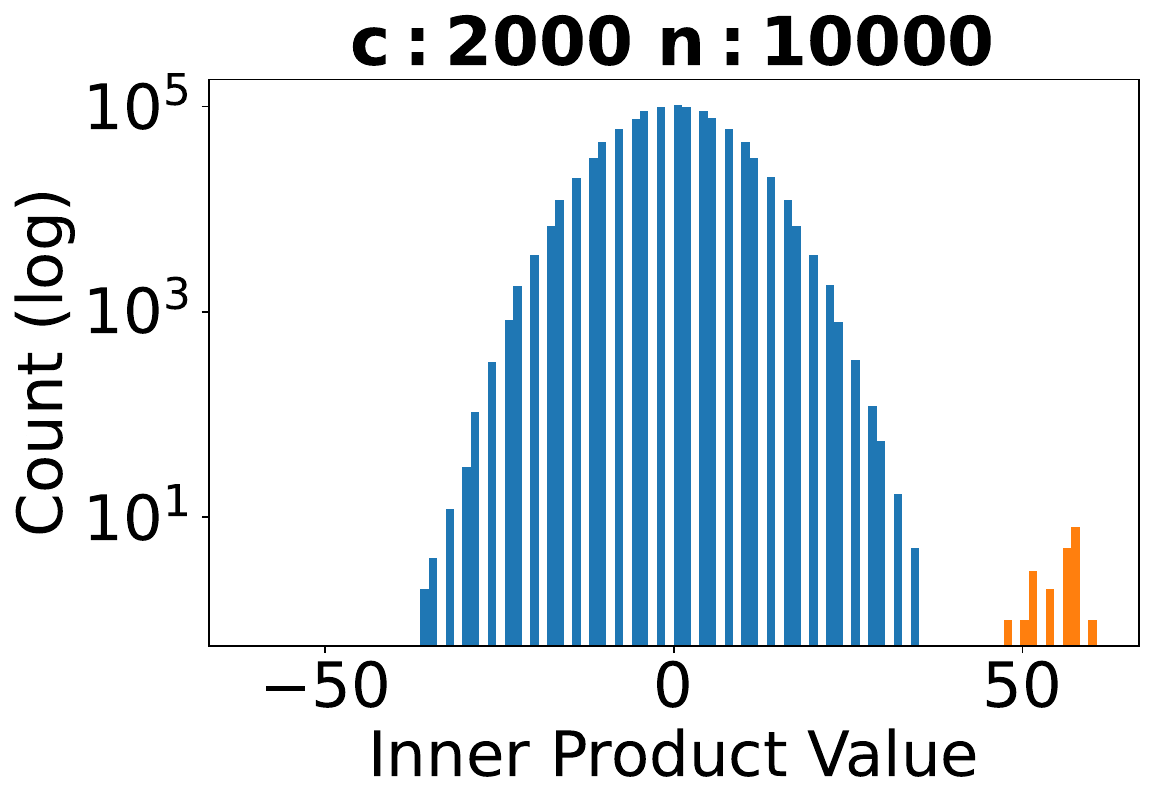}} \\
    \subfloat[]{\includegraphics[width=0.24\textwidth]{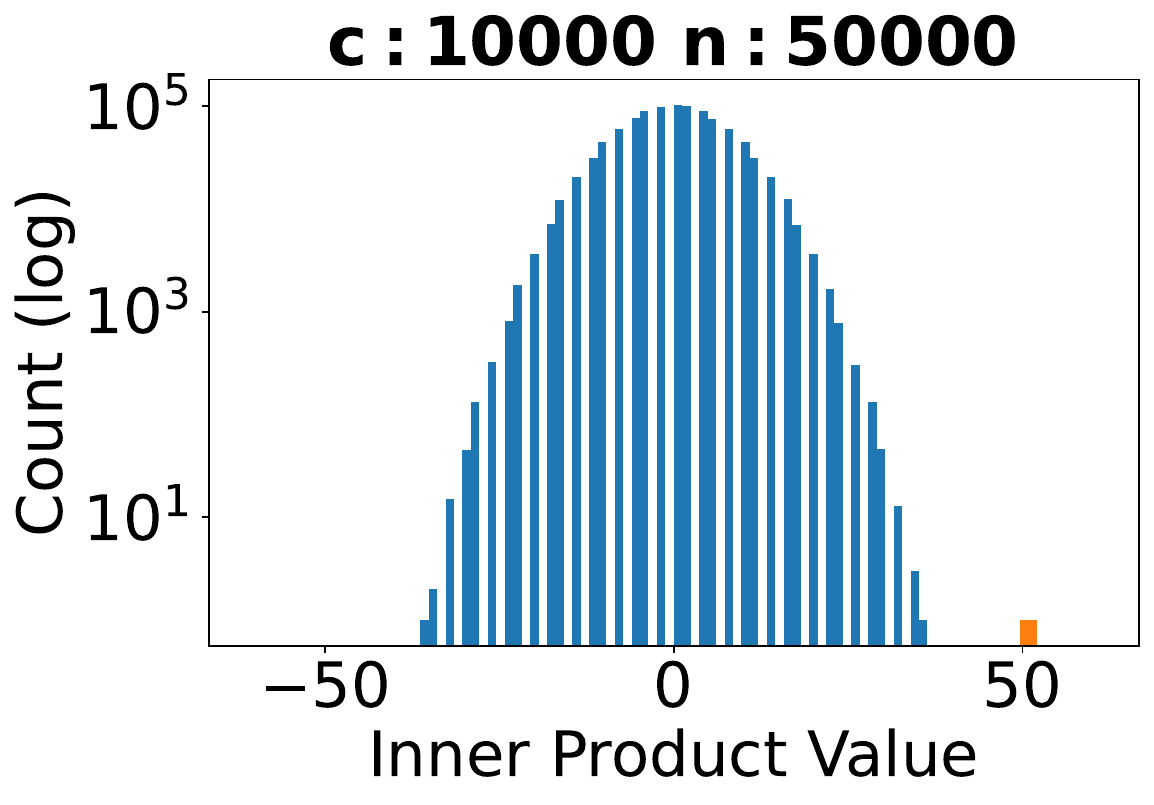}}
    \subfloat[]{\includegraphics[width=0.24\textwidth]{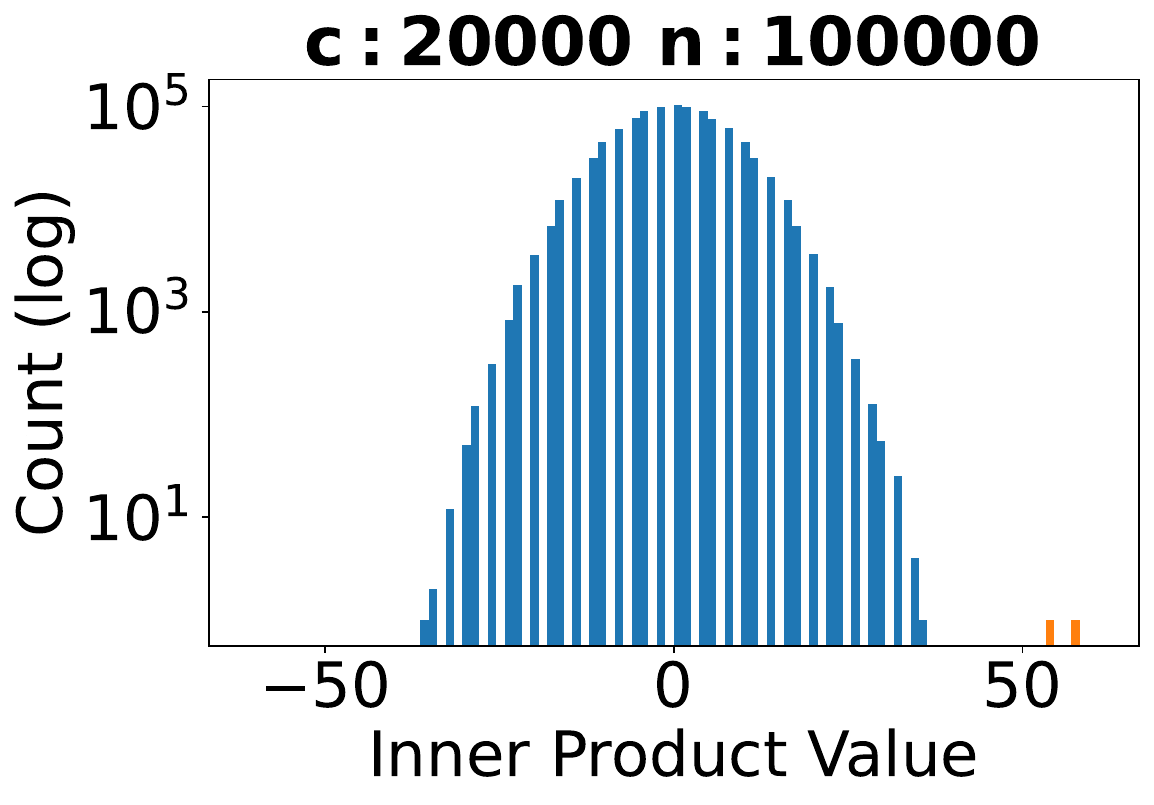}}
    \caption{{\bf Simulation for the heavy inner product identification problem}  
    , where the $y$ values are in logarithmic scale. For each subfigure, we generated two sets $A \subset \{-1,+1\}^{60}$ and $B \subset \{-1,+1\}^{60}$ with $|A|=|B|=n$, where $c$ pairs of them are correlated and others are randomized. The values of $c,n$ are defined in the subfigures. For each entry in the correlated vector pairs, they have $96\%$ chance to be the same and $4\%$ chance to be different. For uncorrelated vectors, we randomize all of them. We randomly pick a pair $(a,b)$ from $A\times B$ and calculate $\langle a, b \rangle$. We repeat this process $1,000,000$ times and get this histogram. We can clearly see that the inner product of uncorrelated vectors (blue) are distributed around $0$, and their absolute value almost don't exceed $50$, which is smaller than the dimension of the vector. Additionally, the inner product of correlated vectors (orange) are distributed around $50$. Our goal is to find these correlated vectors.
}
    \label{fig:distribution}
\end{figure}

\paragraph{Locate the group index containing the correlated vector pairs with $C_{i,j}$.}
We first partition $A$ and $B$ into $h := n^{2/3}$ groups and each group contains $g:=n^{1/3}$ elements. For each $x \in A$, we pick a random value $a^{x} \in \{ -1, 1\}$ independently and uniformly, and for a constant $\tau > 0$ define $r:=\lceil\log_{w}(\tau n^{1 / 3})\rceil$. We define a polynomial $p: \mathbb{R}^{d} \rightarrow \mathbb{R}$ by:
\begin{align*} 
p(z_{1}, \ldots, z_{d})=(z_{1}+\cdots+z_{d})^{r}. 
\end{align*}
For each group pair $(i, j)$ we compute a score $C_{i, j}$ defined as:
\begin{align*}
    C_{i, j}:=\sum_{x \in A_{i}} \sum_{y \in B_{j}} a^{x} \cdot a^{y} \cdot p(x_{1} y_{1}, \ldots, x_{d} y_{d}) .
\end{align*}
If the correlated pair is not in $A_{i} \in \{-1,1\}^{d}$ or $B_{j} \in \{-1,1\}^{d} $, $C_{i, j}$ has expectation $0$.
For sufficiently large constant $\tau$, by the Chebyshev inequality, we have that 
 \begin{align*}
    |C_{i, j}| \leq \tau n^{1 / 3} v^{r} / 3 
 \end{align*}
  with probability at least $3 / 4$. Let $\theta=\tau n^{1 / 3} v^{r} / 3$.
  
If we repeat the process of selecting the $a^{x}$ values for each $x \in A$ independently at random $O(\log n)$ times, whichever top $k$ group pairs $A_{i}, B_{j}$ has $|C_{i, j}| \geq 2 \theta$ most frequently will be the pairs containing the heavy inner product vector pairs with high probability. After identifying the group pairs which contain the heavy inner product vector pairs, it remains a brute force search between $k$ pairs of groups, each containing $O(n^{1/3})$ vectors to output the result.

 \ifdefined\isshort

We defer the proofs into the full version of the paper \cite{full}.
 \else
 
\begin{figure}[!ht]
    \centering
    \includegraphics[width=0.49\textwidth]{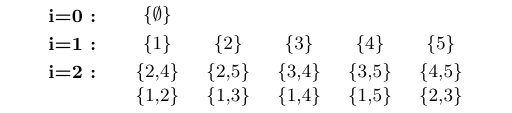}
    \caption{
    {\bf Example of $M_1,\cdots,M_t$, given $d=5$ and $r=2$.}  In this case, $t=16$ as there are $\sum_{i=0}^2\tbinom{5}{i}=16$ distinct subsets $M \subset [d]$ of size less than $2$. }
    \label{fig:subset}
\end{figure}
\fi

\paragraph{Accelerate computing $C_{i, j}$.}
Given $t=\sum_{i=0}^r \tbinom{d}{i}$, let $M_{1}, \ldots, M_{t}$ be an enumeration of all distinct subsets of $[d]$ of size at most $r$
, we can reorganize $C_{i,j}$ as
\begin{align*}
 C_{i, j}
 = & ~ \sum_{s=1}^{t} ( c_{s} \cdot(\sum_{x \in A_{i}} a^{x} \cdot x_{M_{s}}) \cdot(\sum_{y \in B_{j}} a^{y} \cdot y_{M_{s}}))
\end{align*}
for some $c_s, s \in [t]$ that can be computed in $\poly (r) = \poly \log(n)$ time, where  $x_M := \Pi_{i \in M} x_i$ for any $M \subseteq [d]$.

And define the matrices $U, V \in \mathbb{Z}^{h \times t}$ by 
\begin{align*}
    U_{i, s} :=\sum_{x \in A_{i}} a^{x} \cdot x_{M_{s}} \mathrm{~~and~~} V_{i, s} :=c_{s} \cdot U_{i, s}.
\end{align*}
Then we know that the matrix product $C:=U V^{\top} \in \mathbb{Z}^{h \times h}$ 
 is exactly the matrix of the values $C_{i, j}$ we desire and the time complexity of computing the matrix multiplication is 
 \begin{align*}
     \Tmat(n^{2/3}, n^{2 / 3+ o(1)}, n^{2 / 3}) = O(n^{2 \omega / 3+ o(1)}).
 \end{align*}

In addition, we further speed up the calculation of matrix $U$ and $V$ as follows: Let $N_{1}, \ldots, N_{u}$ be an enumeration of all subsets of $[d]$ of size at most $\lceil r / 2\rceil$. For each $i \in[h]$, define the matrices $L^{i}, \tilde{L}^{i} \in \mathbb{Z}^{u \times g}$ by $L_{s, x}^{i} :=x_{N_{s}}$  and $\tilde{L}_{s, x}^{i} :=a^{x} \cdot x_{N_{s}}$.
Then  compute the product $P^{i}:=L^{i} \tilde{L}^{i \top} \in \R^{u \times u}$ and each desired entry $U_{i, s}$ can be found as an entry of the computed matrix $P^{i}$. Computing the entries of the $L^{i}$ matrices naively takes 
\begin{align*}
O(h \cdot u \cdot g \cdot r)=& ~\tilde{O}(n \cdot u) \\ 
= & ~\tilde{O}(n^{4 / 3+ o(1)})
\end{align*}
time, and computing the $P^{i} \in \R^{u \times u}$ takes time 
\begin{align*}
O(h ) \cdot  \Tmat(u, g, u)= O(n^{(2+\omega) / 3+ o(1)} ).
\end{align*}
They are all bounded by $O(n^{2\omega/3 + o(1)})$ time.

\paragraph{Roadmap.} We first discuss several related works in Section~\ref{sec:related_work}. We then introduce several preliminary definitions and lemmas in Section~\ref{sec:prelim}. We present our randomized algorithm design and its main theorem in Section~\ref{sec:randomized_alg}. We further provide a deterministic algorithm for the heavy inner product identification problem in Section~\ref{sec:deterministic}. We conclude our contributions in Section~\ref{sec:conclusion}.
\section{Related Work}\label{sec:related_work}

\paragraph{Correlations on the Euclidean Sphere.}
\cite{c02} proposed a randomized hashing algorithm which rounds the Euclidean sphere to
$\{-1,1\}^{d}$. The hash function chooses a uniformly random hyperplane through the origin, and outputs $1$ or $-1$ depending on which side of the hyperplane a point is on. \cite{a18} proposed an efficient algorithm to solve the light bulb problem~(\cite{prr89}) where given $n$ input vectors in $\{-1, 1\}^{d}$, which are all independently and uniformly random excepted one correlated pair of vectors with high inner product. Our problem is a generalization of theirs, in that our setting has $k$-correlated pairs, with small $k$. 

\paragraph{Learning Sparse Parities with Noise.}
The Light Bulb Problem was first presented by \cite{v88} as a fundamental illustration of a correlational learning problem. The Light Bulb Problem may be generally viewed as a particular instance of a number of other learning theory issues, such as learning sparse parities with noise, learning sparse juntas with or without noise, and learning sparse DNF \cite{fgkp09}. Notably, \cite{fgkp09} demonstrated that all of these more complex learning issues can be reduced to the Light Bulb Problem as well, and they provide best Light Bulb Problem algorithms following from the fastest known algorithms by applying this reduction.

\paragraph{Acceleration via high-dimensional search data structure.}

Finding points in certain geometric query regions can be done efficiently with the use of high-dimensional search data structures. One popular technique is Locality Sensitive Hashing(LSH)~\cite{im98}, which enables algorithms to find the close points. For example, small $\ell_2$ distance~\cite{diim04, ar15, arn17, air18, dirw19} or large inner product between a query vector $q \in \R^{d}$ and a dataset $S \subset \R^{d}$~\cite{sl14, sl15, sl15b}. MONGOOSE~\cite{clp+21} can retrieve neurons with maximum inner products via a learnable LSH-based data structure~\cite{c02} and lazy update framework proposed by~\cite{cls19} in order to achieve forward pass acceleration.

Besides LSH, space partitioning data structures like partition trees~\cite{m92a, m92b, ac09, cha12}, $k$-$d$ trees~\cite{tog17, cha19} can also be leveraged to search the regions.

\paragraph{Sketching}
Sketching is a well-known technique to improve performance or memory complexity~\cite{cw13}. It has wide applications in linear algebra, such as linear regression and low-rank approximation \cite{cw13,nn13,mm13,rsw16,swz17,hlw17,alszz18,swz19_neurips1,swz19_neurips2,djssw19, gsz23, gsyz23}, training over-parameterized neural network~\cite{syz21, szz21, zhasks21}, empirical risk minimization~\cite{lsz19, qszz23}, linear programming \cite{lsz19,jswz21,sy21,lszzz23,y20,gs22}, distributed problems \cite{wz16,bwz16, jll+21}, clustering~\cite{emz21}, generative adversarial networks~\cite{xzz18}, kernel
density estimation~\cite{qrs+22},
tensor decomposition \cite{swz19_soda,dgs23}, trace estimation~\cite{jpwz21}, projected gradient descent~\cite{hmr18, xss21}, matrix sensing~\cite{dls23_sensing, qsz23}, softmax regression and computation \cite{as23,lsz23,dls23,gsy23,ssz23,hjk+23,as23_tensor,kmz23,zhdk23,gswy23,dsxy23}, John Ellipsoid computation \cite{ccly19,syyz22}, semi-definite programming \cite{gs22}, kernel methods~\cite{acw17, akkpvwz20, cy21, swyz21,as23,as23_tensor,gsz23}, adversarial training~\cite{gqsw22}, cutting plane method \cite{jlsw20,jlsz24}, discrepancy \cite{z22}, federated learning~\cite{rpuisbga20, swyz23}, kronecker projection maintenance~\cite{syyz23_dp},  reinforcement learning~\cite{akl17, wzd+20,ssx21},  relational database \cite{qjs+22}.
 
\section{Preliminary}\label{sec:prelim}

\paragraph{Notations.} For any natural number $n$, we use $[n]$ to denote the set $\{1,2,\dots,n\}$.
We use $A^\top$ to denote the transpose of matrix $A$.
 For a probabilistic event $f(x)$, we define ${\bf 1}\{f(x)\}$ such that ${\bf 1}\{f(x)\} = 1$ if $f(x)$ holds and ${\bf 1}\{f(x)\} = 0$ otherwise. 
We use $\Pr[\cdot]$ to denote the probability, and use $\E[\cdot]$ to denote the expectation if it exists. 
For a matrix $A$, we use $\tr[A]$ for trace of $A$. We use $\Tmat(a,b,c)$ to denote the time of multiplying an $a \times b$ matrix with another $b \times c$ matrix.
For a vector $x \in \{-1, 1\}^d$, 
we will use $x_i$ to denote the $i$-th 
entry of $x$ for any $i \in [d]$, and $x_M := \Pi_{i \in M} x_i$ for any $M \subseteq [d]$.
For two sets $A$ and $B$, we use $A \oplus B$ to denote the symmetric difference
of $A$ and $B$.

We give the fast matrix multiplication time notation in the following definition.
\begin{definition}[Matrix Multiplication]\label{def:tmat}
We use $\Tmat(a,b,c)$ to denote the time of multiplying an $a \times b$ matrix with another $b \times c$ matrix. We use $\omega$ to denote the number such that $\Tmat(n, n, n) = n^{\omega}$. Usually $\omega$ is called the exponent of matrix multiplication.  
Currently, $\omega \approx 2.373$~\cite{w12, l14}. 
\end{definition}


 \ifdefined\isshort

We defer some well-known inequality into the full version of the paper \cite{full}.
 \else
We will use Hoeffding bound and the Chebyshev's inequality as a probability tool.
\begin{lemma}[Hoeffding bound~\cite{h63}]\label{lema:hoeffding}
Let $X_{1}, \cdots, X_{n}$ denote $n$ independent bounded variables in $[a_{i}, b_{i}] .$ Let $X=\sum_{i=1}^{n} X_{i}$, then we have
\begin{align*}
    \Pr[|X-\mathbb{E}[X]| \geq t] \leq 2 \exp (-\frac{2 t^{2}}{\sum_{i=1}^{n}(b_{i}-a_{i})^{2}})
\end{align*}

\end{lemma}

\begin{lemma}[Chebyshev's inequality]\label{lem:chebyshev}
Let  $X$  be a random variable with expectation  $\mu$  and standard deviation  $\sigma$. Then for all  $c > 0$, we have:
\begin{align*}
    \Pr[|X - \mu| \geq c \sigma] \leq \frac{1}{c^2}
\end{align*}
\end{lemma}

Next, we present the definition of hash function here. Usually, hashing is used in data structures for storing and retrieving data quickly with high probability. 

\begin{definition}[Hash function]\label{def:hash}
A hash function is any function that projects a value from a set of arbitrary size to a value from a set of a fixed size. Let $U$ be the source set and $D$ be the target set, we can define the hash function $h$ as:
$
    h: U \to D
$
\end{definition}

We will need partition the input sets $A \subset \{-1, 1\}^{d}$ and $B \subset \{-1, 1\}^{d}$ into groups using the pairwise independent hashing function. The following definition shows the collision probability of pairwise independent hash functions.
\begin{definition}[Pairwise independence]\label{lem:hash_collision}
A family $\mathcal{H} = \{ h : U \rightarrow R\}$ is said to be pairwise independent if for any two distinct elements $x_1 \neq x_2 \in U$, and any two (possibly equal) values $y_1, y_2 \in R$ such that:
\begin{align*}
    \Pr_{h \in \mathcal{H}}[h(x_1) = y_1 \mathrm{~~and~~} h(x_2) = y_2] = \frac{1}{|R|^2}
\end{align*}
where $R$ is a set with finite elements.
\end{definition}

\fi

\section{Algorithm for large inner product between two sets}\label{sec:randomized_alg}

\paragraph{Roadmap.} In this section, we first present our main result (Theorem~\ref{thm:main}) for the problem defined in Definition~\ref{def:main_problem} and its corresponding algorithm implementation \ifdefined\isshort Algorithm 1 and 2 in full version \cite{full}. \else in Algorithm~\ref{alg:light_bulb} and Algorithm~\ref{alg:light_bulb2}. \fi
In Section~\ref{sec:correctness} we present the correctness lemma and proof for our algorithm. In Section~\ref{sec:partition_collision} we present the lemma and proof for our partition based on pairwise independent hash function. We present the time complexity analysis our algorithm in Section~\ref{sec:running_time}.

 \ifdefined\isshort

We defer Our Algorithms into the full version of the paper \cite{full}.
 \else

\begin{algorithm}[!ht]\caption{Algorithm for finding the heavy group pairs and finding the heavy inner product pair within one group pair. 
} 
\label{alg:light_bulb}
\begin{algorithmic}[1]
\State {\bf data structure}
\State {\bf members}
\State \hspace{4mm} $\theta \in \R$ \Comment{The threshold for group pair score $C_{i,j}$}
\State {\bf end members}
\State
\Procedure{FindHeavyGroupPairs}{$ \{A_1, \cdots, A_h\} \subset \{-1, 1\}^d, \{B_1, \cdots, B_h\} \subset \{-1, 1\}^d,$ $n \in \mathbb{N}_{+}, t \in \mathbb{N}_{+} $ 
}
\State $R \gets \emptyset$ 
\For{$\ell = 1 \rightarrow 10 \log(n)$}
    \For{$i = 1 \rightarrow h$} 
        \For{$j = 1 \rightarrow h$} \Comment{Pick $a^{x} , b^y \in\{-1,1\}$ at uniformly random}
                \State  $C_{i, j}= \sum_{s=1}^{t}[c_{s} \cdot(\sum_{x \in A_{i}} a^{x} \cdot x_{M_{s}}) \cdot(\sum_{y \in B_{j}} a^{y} \cdot y_{M_{s}})]$   
                \If{$C_{i, j} \geq 2 \theta$}
                    \State $R.\textsc{append}((i, j))$ 
                \EndIf
        \EndFor
    \EndFor
\EndFor \Comment{$|R|=10 k \log(n)$}
\State \Return $R$
\EndProcedure
\State
\Procedure{SolveOneGroupPair}{$A_i \subset \{-1 , 1\}^d,  B_j \subset \{-1 , 1\}^d, g \in \mathbb{N}_{+}$}
\For{$q = 1 \rightarrow g$} \Comment{Brute force search within each group pair}
        \For{$l = 1 \rightarrow g$}
            \If{$\langle A_{i, q}, B_{j, l} \rangle \geq \rho d$ }
                \State \Return $(i \cdot g + q, j \cdot g + l)$ \Comment{Find the heavy inner product vector pair within group $A_i$ and $B_j$}
            \EndIf
        \EndFor
    \EndFor
\EndProcedure

\State {\bf end data structure}
\end{algorithmic}
\end{algorithm}
\fi

\ifdefined\isshort

 \else
\begin{algorithm}[!ht]\caption{Algorithm for heavy inner product between two sets problem. 
} 
\label{alg:light_bulb2}
\begin{algorithmic}[1]
\State {\bf data structure}
\Procedure{FindCorrelated}{$A  \subset \{-1, 1\}^{d}, B \subset \{-1, 1\}^{d}, n \in \mathbb{N}_{+}, k \in \mathbb{N}_{+}, \tau \in \R_{+}, w \in \R_{+}, \rho \in (0,1), \delta \in \R$} 
\State Choose two pairwise independent hash functions $h_A , h_B : \{-1, 1\}^{d} \rightarrow [n^{2/3}]$
\State $h \gets n^{2/3}, g \gets n^{1/3}$, $v \gets \delta(w / \kappa) \log n$, $r \gets \lceil\log_{w}(\tau n^{1 / 3})\rceil$, $\theta \gets { \tau n^{1/3} v^{r}}/ {3}$
\State $t \gets \sum_{i=0}^{r}{d \choose i}$
\State Partition $A$ into $h$ groups $A_1, \cdots, A_h$ and $B$ into $B_1, \cdots, B_h$. Each contains $g$ vectors.
\State $R \gets \textsc{FindHeavyGroupPairs}( \{A_1, \cdots, A_h\}$, \\$\{B_1, \cdots, B_h\}, n, t)$ \Comment{Time complexity is  $O(n^{2 \omega / 3+o(1)})$, Algorithm~\ref{alg:light_bulb}}
\State $F \gets \emptyset$ 
\For{$\ell = 1 \rightarrow |R|$}
    \State $(i,j) \gets R_{\ell}$ 
    \State $p \gets \textsc{SolveOneGroupPair}(A_i, B_j, g)$ \Comment{Algorithm~\ref{alg:light_bulb}}
    \State $F.\textsc{Append}(p)$ 
\EndFor
\State \Return $F$
\EndProcedure

\State {\bf end data structure}
\end{algorithmic}
\end{algorithm}
\fi


\begin{theorem}[Formal version of Theorem~\ref{thm:main_prelim}]\label{thm:main}
Given two sets $A \subset \{-1,+1\}^d$ and $B \subset \{-1,+1\}^d$ with $|A|=|B|=n$, 
which are all independently and uniformly random except for $k$ vector pairs $\{(a_1, b_1),\cdots, (a_k, b_k) \} \subset A \times B$ such that $\forall i \in [k], \langle a_i,b_i\rangle \geq \rho \cdot d$, for some $0 < \rho \leq 1$. 
For every $\epsilon, \rho>0$, there is a $c_0>0$ such that for correlation $\rho$, we can find the $k$ pairs $\{(a_1, b_1),\cdots, (a_k, b_k) \}$ in randomized time 
\begin{align*}
   O(n^{2 \omega / 3+ o(1)})
\end{align*}
 whenever $d=c_0 \log n$ with probability at least $1 - 1/n^{10}- k^2/n^{2/3}$. 
\end{theorem}
\begin{proof}
The correctness of the theorem follows by Lemma~\ref{lem:correctness_light_bulb}, and the proof of running time follows by Lemma~\ref{lem:time_light_bulb}. 

\end{proof}


\subsection{Correctness of \textsc{FindCorrelated} Algorithm}\label{sec:correctness}
We first present the following lemma to prove the correctness of \textsc{FindCorrelated} in \ifdefined\isshort Algorithm 2 in full version \cite{full} \else Algorithm~\ref{alg:light_bulb2} \fi.

\begin{lemma}[Correctness]\label{lem:correctness_light_bulb}
For problem in Defnition~\ref{def:main_problem}, and for every $\epsilon, \rho>0$, there is a 
$c_0>0$ such that for correlation $\rho$ we can find the $k$ pairs $\{(a_1, b_1),\cdots, (a_k, b_k) \}$  whenever $d=c_0 \log n$ with probability at least $1 - 1/n^{10}- k^2/n^{2/3}$. 
\end{lemma}

 \ifdefined\isshort

We defer the proofs into the full version of the paper \cite{full}.
 \else
 
\begin{proof}

For two constants $\delta, w>0$ to be determined, we will pick $c_0=\delta w^{2} / \rho^{2}$. Let $A, B \subset\{-1,1\}^{d}$ be the set of input vectors, and let $\wt{x} \in A, \wt{y} \in B$
denote the heavy inner product pairs which we are trying to find.

For distinct $x \in A, y \in B$ other than the heavy inner product pairs, the inner product $\langle x, y\rangle$ is a sum of $d$ uniform independent $\{-1,1\}$ values. 

Let $v:=\delta(w / \kappa) \log n$. For large enough $\delta$, by a Hoeffding bound stated in Lemma~\ref{lema:hoeffding} we have:
\begin{align*}
     & ~ \Pr[|\langle x, y \rangle - \E[\langle x, y \rangle]| \geq v] \\ 
    = & ~ \Pr[|\langle x, y \rangle| \geq v] \\
    \leq & ~ 2\exp(-\frac{v^2}{2n}) \\
    = & ~ 2\exp(-\frac{\delta^2 w^2}{2 n \kappa^2 \log^2(n)}) \\
    \leq & ~ 1/n^{13} 
\end{align*}
where the first step follows that $\E[\langle x, y \rangle] = 0$,
the second step follows Hoeffding bound (Lemma~\ref{lema:hoeffding}), 
the third step comes from $v=\delta(w / \kappa) \log n$,
and the fourth step follows that $\delta^2 \geq \frac{26 n \kappa^2 \log^3(n)}{w^2}$. 
Therefore, we have $|\langle x, y\rangle| < v$ with probability at least $1-1 / n^{13}$.

Hence, by a union bound over all $n^2 - k$ pairs of uncorrelated vectors 
, we have $|\langle x, y\rangle| < v$ for all such $x, y$ with probability at least $1-1 / n^{11}$. We assume henceforth that this is the case. Meanwhile, $\langle \wt{x}, \wt{y}\rangle \geq \rho d=w v$.

Arbitrarily partition $A \subset\{-1,1\}^{d}$ into $h:=n^{2 / 3}$ groups $A_{1}, \ldots, A_{h} \subset\{-1,1\}^{d}$ of size $g:=n / h=n^{1 / 3}$ per group, and partition $B \subset\{-1,1\}^{d}$ into $h$ groups $B_{1}, \ldots, B_{h} \subset\{-1,1\}^{d}$ of size $g$ per group too. According to Lemma~\ref{lem:group_collision}, we condition on that none of the group pair contains more than one heavy inner product vector pair.

For each $x \in A$, our algorithm picks a value $a^{x} \in\{-1,1\}$ independently and uniformly at random. For a constant $\tau>0$ to be determined, let 
\begin{align*}
  r:=\lceil\log_{w}(\tau n^{1 / 3})\rceil ,
\end{align*}
and define the polynomial $p: \mathbb{R}^{d} \rightarrow \mathbb{R}$ by
\begin{align*} 
p(z_{1}, \ldots, z_{d})=(z_{1}+\cdots+z_{d})^{r}. 
\end{align*}

Our goal is, for each $(i, j) \in[h]^{2}$, to compute the value
\begin{align}\label{eq:compute_C}
    C_{i, j}:=\sum_{x \in A_{i}} \sum_{y \in B_{j}} a^{x} \cdot a^{y} \cdot p(x_{1} y_{1}, \ldots, x_{d} y_{d}) .
\end{align}

First we explain why computing $C_{i, j}$ is useful in locating the groups which contain the correlated pair.

Denote 
\begin{align*}
p(x, y):=p(x_{1} y_{1}, \ldots, x_{d} y_{d}). 
\end{align*}

Intuitively, $p(x, y)$ is computing an amplification of $\langle x, y\rangle$. $C_{i, j}$ sums these amplified inner products for all pairs $(x, y) \in A_{i} \times B_{j}$. We can choose our parameters so that the amplified inner product of the correlated pair is large enough to stand out from the sums of inner products of random pairs with high success probability.

Let us be more precise. Recall that for uncorrelated $x, y$ we have $|\langle x, y\rangle| \leq v$, and hence $|p(x, y)| \leq v^{r}$ 

Similarly, we have 
\begin{align*}
  |p(\wt{x}, \wt{y})| 
\geq ~ (w v)^{r} 
\geq  ~  \tau n^{1 / 3} v^{r}. 
\end{align*}
where the first step comes from $(\wt{x}, \wt{y})$ is the correlated pair 
, and the second step comes from $r=\lceil\log_{w}(\tau n^{1 / 3})\rceil $

For $x \in A, y \in B$, define 
\begin{align*}
a^{(x, y)}:=a^{x} \cdot a^{y}.
\end{align*}

Notice that, 
\begin{align*}
   C_{i, j}=\sum_{x \in A_{i}, y \in B_{j}} a^{(x, y)} \cdot p(\langle x, y\rangle),
\end{align*}
where the $a^{(x, y)}$ are pairwise independent random $\{-1,1\}$ values.

We will now analyze the random variable $C_{i, j}$ where we think of the vectors in $A$ and $B$ as fixed, and only the values $a^{x}$ as random. Consider first when the correlated pair are not in $A_{i} \in \{-1,1\}^{d}$ and $B_{j} \in \{-1,1\}^{d} $. Then, $C_{i, j}$ has mean 0 , and (since variance is additive for pairwise independent variables) $C_{i, j}$ has variance at most 
\begin{align*}
   |A_{i}| \cdot|B_{j}| \cdot \max_{x \in A_{i}, y \in B_{j}}|p(\langle x, y\rangle)|^{2} \leq n^{2 / 3} \cdot v^{2 r}. 
\end{align*}

 For a constant $\tau = \sqrt{10} $, we have that 
 \begin{align*}
 \Pr[|C_{i, j}| \leq \tau/3 \cdot n^{1 / 3} v^{r} ] 
    \geq  ~ \frac{9}{\tau^2} 
    = ~9/10
 \end{align*}
where the first step follows that $C_{i,j}$ has mean $0$, $\sigma = n^{1 / 3} v^{r}$ and the Chebyshev inequality from Lemma~\ref{lem:chebyshev}, and the second step comes from $\tau = \sqrt{10} $.

Let $\theta=\tau n^{1 / 3} v^{r} / 3$, so $|C_{i, j}| \leq \theta$ with probability at least $9/10$.
Meanwhile, if $\wt{x} \in A_{i}$ and $\wt{y} \in B_{j}$, then $C_{i, j}$ is the sum of $a^{(\wt{x}, \wt{y})} p(\langle \wt{x}, \wt{y}\rangle)$ and a variable $C'$ distributed as $C_{i, j}$ was in the previous paragraph.

Hence, since $|p(\langle \wt{x}, \wt{y}\rangle)| \geq \tau n^{1 / 3} v^{r}=3 \theta$, and $|C'| \leq \theta$ with probability at least $9/10$, we get by the triangle inequality that $|C_{i, j}| \geq 2 \theta$ with probability at least $9/10$.

Hence, if we repeat the process of selecting the $a^{x}$ values for each $x \in A$ independently at random $10 \log n$ 
times, and check all group pairs $(A_{i}, B_{j})$ having $|C_{i, j}| \geq 2 \theta$ in a brute force manner to locate the heavy inner product vector pairs. For each group pair brute force check, the time complexity is $O(n^{2/3})$ and there are $10 k \log n$ 
group pairs to check. The failure probability is $1 - k/n^{11}$.

In all, by a union bound over all possible errors and the probability that any group pair contains more than one heavy inner product vector pairs from Lemma~\ref{lem:group_collision}, this will succeed with probability at least 
\begin{align*}
 ~ 1 - (k+1)/n^{11}- k^2 / n^{2/3} 
\geq  ~ 1- 1/n^{10} - k^2/n^{2/3}
\end{align*} 
\end{proof}

\fi

\subsection{Partition Collision Probability}\label{sec:partition_collision}
In the following lemma, we prove that if there are $k$ heavy inner product vector pairs in $A$ and $B$, the probability of none of the group pair contains more than one heavy inner product vector pair is $1 - k^2 / n^{2/3}$.
\begin{lemma}\label{lem:group_collision}
With two pairwise independent hash function $h_A$ and $h_B$, we partition two sets $A, B \subset \{-1, 1\}^{d}$ into $h = n^{2/3}$ groups respectively. Suppose there are $k$ heavy vector inner product pairs $\{(a_1, b_1), \cdots, (a_k, b_k)\} \subset A \times B$, the probability where none of the group pair contains more than one heavy inner product vector pair is $1 - k^2 / n^{2/3}$.
\end{lemma}

 \ifdefined\isshort

We defer the proofs into the full version of the paper \cite{full}.
 \else
 
\begin{proof}
If any two vector pairs $(a_x, b_x), (a_y, b_y) \in \{(a_1, b_1), \cdots, (a_k, b_k)\} \subset A \times B$ collide in a group pair $(A_i, B_j)$, according to the collision property of hash function $h_A$ and $h_B$, the probability is ${n^{-2/3}}$. By union bound across all $k^2$ possible heavy inner product vector pairs, we know that the probability of none of any two vector pairs in $\{(a_1, b_1), \cdots, (a_k, b_k)\} $ collide into the same group is $1 - k^2 / n^{2/3}$.

This completes the proof.
\end{proof}
\fi


\subsection{Running Time of \textsc{FindCorrelated} Algorithm}\label{sec:running_time}

We split the time complexity of \textsc{FindCorrelated} in \ifdefined\isshort Algorithm 2 in the full version \cite{full} \else Algorithm~\ref{alg:light_bulb2} \fi into two steps and analyze the first step in Lemma~\ref{lem:time_step_1}.

\subsubsection{Running time of finding heavy group pairs} 
\begin{lemma}[Time complexity of \textsc{FindHeavyGroupPairs}]\label{lem:time_step_1} 
For every $\epsilon, \rho>0$, \textsc{FindHeavyGroupPairs} in \ifdefined\isshort Algorithm 1 in the full version of the paper \cite{full} \else Algorithm~\ref{alg:light_bulb} \fi for heavy inner product between two sets problem in Definition~\ref{def:main_problem} of correlation $\rho$ runs in 
\begin{align*} 
O(n^{2 \omega / 3+o(1)})
\end{align*}
time with a $c_0>0$ and $d = c_0 \log n$ to find all the group pairs which contains the heavy inner product vector pairs. 
\end{lemma}

 \ifdefined\isshort

We defer the proofs into the full version of the paper \cite{full}.
 \else
 
\begin{proof}

Before computing $C_{i,j}$, we will rearrange the expression for $C_{i, j}$ into one which is easier to compute. Since we are only interested in the values of $p$ when its inputs are all in $\{-1,1\}$, we can replace $p$ with its multilinearization $\hat{p}$. 

Let $M_{1}, \ldots, M_{t}$ be an enumeration of all subsets of $[d]$ of size at most $r$, thus we know $t$ can be calculated as follows:
\begin{align*}
    t=\sum_{i=0}^{r}{d \choose i}.
\end{align*}

Then, there are coefficients $c_{1}, \ldots, c_{t} \in \mathbb{Z}$ such that 
\begin{align*}
\hat{p}(x)=\sum_{s=1}^{t} c_{s} x_{M_{s}}.
\end{align*}

Rearranging the order of summation of Eq.~\eqref{eq:compute_C}  
, we see that we are trying to compute
\begin{align}\label{eq:C_i_j}
 C_{i, j}
 = & ~ \sum_{s=1}^{t} \sum_{x \in A_{i}} \sum_{y \in B_{j}} a^{x} \cdot a^{y} \cdot c_{s} \cdot x_{M_{s}} \cdot y_{M_{s}} \notag \\
 = & ~ \sum_{s=1}^{t} ( c_{s} \cdot(\sum_{x \in A_{i}} a^{x} \cdot x_{M_{s}}) \cdot(\sum_{y \in B_{j}} a^{y} \cdot y_{M_{s}}))
\end{align}
In order to compute $C_{i, j}$, we first need to compute the coefficients $c_{s}$. Notice that $c_{s}$ depends only on $|M_{s}|$ and $r$. We can thus derive a simple combinatorial expression for $c_{s}$, and hence compute all of the $c_{s}$ coefficients in  $\poly(r)= \mathrm{polylog}(n)$ time. Alternatively, by starting with the polynomial $(z_{1}+\cdots+z_{d})$ and then repeatedly squaring then multilinearizing, we can easily compute all the coefficients in $O(t^{2} \mathrm{polylog}(n))$ time; this slower approach is still fast enough for our purposes.

Define the matrices $U, V \in \mathbb{Z}^{h \times t}$ by 
\begin{align*}
    U_{i, s} :=\sum_{x \in A_{i}} a^{x} \cdot x_{M_{s}} \mathrm{~~and~~} V_{i, s} :=c_{s} \cdot U_{i, s}.
\end{align*}
 Notice from Eq.~\eqref{eq:C_i_j} that the matrix product $C:=U V^{\top} \in \mathbb{Z}^{h \times h}$ is exactly the matrix of the values $C_{i, j}$ we desire.

 A simple calculation (see Lemma~\ref{lem:bound_t_and_u} below) shows that for any $\epsilon>0$, we can pick a sufficiently big constant $w>0$ such that $t=O(n^{2 / 3+ o(1)})$.

 Since $h=O(n^{2 / 3})$, if we have the matrices $U, V$, then we can compute this matrix product in 
\begin{align*}
    \Tmat(n^{2/3}, n^{2 / 3+o(1)}, n^{2 / 3}) = O(n^{2 \omega / 3+o(1)})
\end{align*}
 time, completing the algorithm.

Unfortunately, computing the entries of $U$ and $V$ naively would take $\Omega(h \cdot t \cdot g)=\Omega(n^{5 / 3})$ time, which is slower than we would like. 


Let $N_{1}, \ldots, N_{u}$ be an enumeration of all subsets of $[d]$ of size at most $\lceil r / 2\rceil$. For each $i \in[h]$, define the matrices $L^{i}, \tilde{L}^{i} \in \mathbb{Z}^{u \times g}$ (whose columns are indexed by elements $x \in A_{i}$ ) by 
\begin{align*}
    L_{s, x}^{i} :=x_{N_{s}}  \mathrm{~~and~~}  \tilde{L}_{s, x}^{i} :=a^{x} \cdot x_{N_{s}}.
\end{align*}
Then  compute the product $P^{i}:=L^{i} \tilde{L}^{i \top} \in \R^{u \times u}$. We can see that
\begin{align*}
   P_{s, s'}^{i}=\sum_{x \in A_{i}} a^{x} \cdot x_{N_{s} \oplus N_{s'}}, 
\end{align*}
where $N_{s} \oplus N_{s'}$ 
is the symmetric difference of $N_{s}$ and $N_{s'}$. Since any set of size at most $r$ can be written as the symmetric difference of two sets of size at most $\lceil r / 2\rceil$, each desired entry $U_{i, s}$ can be found as an entry of the computed matrix $P^{i}$. Similar to our bound on $t$ from before (see Lemma~\ref{lem:bound_t_and_u} below), we see that for big enough constant $w$, we have $u=O(n^{1 / 3+o(1)})$. Computing the entries of the $L^{i}$ matrices naively takes only 
\begin{align*}
O(h \cdot u \cdot g \cdot r)
= ~ & \tilde{O}(n \cdot u) \\ 
= ~ & \tilde{O}(n^{4 / 3+o(1)})
\end{align*}
time, and then computing the products $P^{i} \in \R^{u \times u}$ 
takes time 
\begin{align*}
O(h ) \cdot  \Tmat(u, g, u)= O(n^{(2+\omega) / 3+o(1)} ).
\end{align*}
 Both of these are dominated by $O(n^{2 \omega / 3+o(1)})$. 
This completes the proof.
\end{proof}

\fi

We will need the following lemma to prove the time complexity of our algorithm.  
\begin{lemma}\label{lem:bound_t_and_u}
Let $d=O(w^{2} \log (n))$ and $n = |A| = |B|$ be the vector dimension and number of vectors in Theorem~\ref{thm:main}, and let 
\begin{align*}
    r := \lceil\log_{w}(\tau n^{1 / 3})\rceil, u := {d \choose r/2}, t := {d \choose r}
\end{align*}
For every $\epsilon>0$,  and there is a $w>0$ such that we can bound
\begin{align*}
  t=O(n^{2 / 3+o(1)}), u=O(n^{1 / 3+o(1)}).   
\end{align*}

Lemma \ref{lem:bound_t_and_u} implies that for any $\epsilon >0$, we can find a sufficiently large $w$ to bound the time needed to compute $UV^\top$ and $P^{i}:=L^{i} \tilde{L}^{i \top}$.
 
\end{lemma}

 \ifdefined\isshort

We defer the proofs into the full version of the paper \cite{full}.
 \else
 
\begin{proof}
Recall that $r=\log_{w}(O(n^{1 / 3}))$. 
We can upper bound $t$ as follows:
\begin{align*}
 t \leq & ~ (r+1) \cdot {d \choose r} \\
 \leq & ~ (r+1) \cdot(e d / r)^{r} \\
 \leq & ~ O(w^{2} \log (w))^{\log _{k}(O(n^{1 / 3}))} \\
 = & ~ n^{2 / 3+O(\log \log (w) / \log (w))}
\end{align*}
where the first step follows from $t=\sum_{i=0}^{r}{d \choose i} \leq  (r+1) \cdot {d \choose r}$,
the second step follows that ${d \choose r} \leq (ed/r)^{r}$,
the third step comes from the value of $d$ and $r$,
and the final step comes from $a^{\log_{a} n} = n$.
For any $\epsilon>0$ we can thus pick a sufficiently large $w$ so that $t \leq O(n^{2 / 3+o(1)})$. 

We can similarly upper bound ${d \choose r/2}$ $\leq O(n^{1 / 3+o(1)})$ which implies our desired bound on $u$.

This completes the proof.
\end{proof}
\fi

\subsubsection{Running time of solving heavy group pairs}  

Then we analyze the time complexity of solving heavy group pairs in Lemma~\ref{lem:time_step_2}.
\begin{lemma}[Time complexity of solving heavy group pairs  ]\label{lem:time_step_2} 
 For every $k, \rho>0$, solving the heavy group pairs in \ifdefined\isshort Algorithm 1 in full version of the paper \cite{full} \else Algorithm~\ref{alg:light_bulb} \fi for heavy inner product between two sets problem in Definition~\ref{def:main_problem} 
 of correlation $\rho$ runs in 
 $
    \tilde{O}(k \cdot n^{2 / 3}) 
$
 time.
\end{lemma}
 \ifdefined\isshort

We defer the proofs into the full version of the paper \cite{full}.
 \else
 
\begin{proof}

Recall that for finding the heavy group pairs of \textsc{FindCorrelated} algorithm, we use two pairwise independent hash function $h_A$ and $h_B$ to partition $A$ into $h:=n^{2 / 3}$ groups $A_{1}, \ldots, A_{h}$ of size $g:=n / h=n^{1 / 3}$ per group, and partition $B$ into $h$ groups $B_{1}, \ldots, B_{h}$ of size $g$ per group too. 

Because we have found $10 k \log n$ group pairs $\{(A_{i,1}, B_{j,1}), \cdots, (A_{1,k}, B_{j, k})\} $ 
which contain the $k$ heavy inner product vector pairs at the end of \textsc{FindHeavyGroupPairs} with high probability,
a brute force within these  group pairs of $O(n^{1 / 3})$ vectors can find the correlated pair in $\tilde{O}(k \cdot n^{2 / 3})$ time. This completes the proof.
\end{proof}
\fi

\subsubsection{Overall running time}
With the above two lemmas in hand, we can obtain the overall time complexity of \textsc{FindCorrelated} algorithm in Lemma~\ref{lem:time_light_bulb}.
\begin{lemma}[Time complexity]\label{lem:time_light_bulb}
For problem in Definition~\ref{def:main_problem}, and for every 
$\rho>0$, there is algorithm (\ifdefined\isshort Algorithm 2 in full version of the paper \cite{full} \else Algorithm~\ref{alg:light_bulb2} \fi) such that for correlation $\rho$ we can find the $k$ pairs $\{(a_1, b_1),\cdots, (a_k, b_k) \}$  whenever $d=c_0 \log n$  in 
\begin{align*} 
O(n^{2 \omega / 3+o(1)})
\end{align*}
time. 
\end{lemma}

 \ifdefined\isshort

We defer the proofs into the full version of the paper \cite{full}.
 \else
 
\begin{proof}
The proof follows by Lemma~\ref{lem:time_step_1} and Lemma~\ref{lem:time_step_2}. We can obtain the overall time complexity by:
\begin{align*}
    \mathrm{total~time}
    = & ~ O(n^{2 \omega / 3+o(1)}) + \wt{O}(k \cdot n^{2 / 3}) \\
    = & ~ O(n^{2 \omega / 3+o(1)})
\end{align*}
where the first step comes from the time complexity of \textsc{FindHeavyGroupPairs} and solving all heavy group pairs in Lemma~\ref{lem:time_step_1} and Lemma~\ref{lem:time_step_2},
and the second step follows that 
\begin{align*}
   O(n^{2 \omega / 3+o(1)}) > \wt{O}(k \cdot n^{2 / 3}). 
\end{align*}

This completes the proof.
\end{proof}
\fi

\section{Deterministic Algorithm}\label{sec:deterministic}

We now present a deterministic algorithm for the heavy inner product between two sets problem. Each is a slight variation on the algorithm from the previous section. 

\ifdefined\isshort

We defer Our Algorithms into the full version of the paper \cite{full}.
 \else
\begin{algorithm*}[!ht]\caption{Deterministic algorithm for heavy inner product between two sets problem.
}
\label{alg:light_bulb_deterministic}
\begin{algorithmic}[1]
\Procedure{FindRandomBits}{$ A \subset \{-1, 1\}^d, B \subset \{-1, 1\}^d, \rho \in (0,1), n\in \mathbb{N}_{+}$}
\State Initialize an empty set $\wt{A}$
\For{$i = 1 \rightarrow n$}
    \State $r \gets 0$
    \For{$j = 1 \rightarrow n$}
        \If{$\langle a_{i}, b_j \rangle \geq \rho d$} 
            \State $r \gets 1$ \Comment{If the inner product of vector $a_i$ and $b_j$ surpass the threshold, we set $r := 1$.}
        \EndIf
    \EndFor
    \If{$r = 0$}
    \State $\wt{A}.\textsc{append}(a_i)$ \Comment{If $a_i$ is a random vector, we include it in the set $\wt{A}$.}
    \EndIf
    \If{$|\wt{A}| \geq \log n$}
        \State Break \Comment{We have collected enough random vector for next step.}
    \EndIf
\EndFor
\State \Return $\wt{A}$
\EndProcedure
\Procedure{FindCorrelated}{$A  \subset \{-1, 1\}^{d}, B \subset \{-1, 1\}^{d}, n \in \mathbb{N}_{+}, k \in \mathbb{N}_{+}, \tau \in \R_{+}, w \in \R_{+}, \rho \in (0,1), \delta \in \R $ 
}

\State Choose two pairwise independent hash functions $h_A , h_B : \{-1, 1\}^{d} \rightarrow [n^{2/3}]$
\State $h \gets n^{2/3}, g \gets n^{1/3}$, $v \gets \delta(w / \kappa) \log n$, $r \gets \lceil\log_{w}(\tau n^{1 / 3})\rceil$, $\theta \gets { \tau n^{1/3} v^{r}}/ {3}$
\State $t \gets \sum_{i=0}^{r}{d \choose i}$
\State Partition $A$ into $h$ groups $A_1, \cdots, A_h$ and $B$ into $h$ groups $B_1, \cdots, B_h$. Each contains $g$ vectors.
\State $\wt{A} \gets  \textsc{FindRandomBits}(A, B, \rho, n)$ \Comment{Time complexity is $O(n \log ^{2}(n))$}
\State $R \gets \textsc{FindHeavyGroupPairs}( \{A_1, \cdots, A_h\},$ $\{B_1, \cdots, B_h\}, n, t)$ 
\State \Comment{Time complexity is  $O(n^{2 \omega / 3+o(1)})$. The $a^x$ and $a^y$ generation uses the random bits from $\wt{A}$. Algorithm~\ref{alg:light_bulb}}
\State $F \gets \emptyset$ \Comment{We use set $F$ to store all correlated vector pairs}
\For{$\ell = 1 \rightarrow |R|$}
    \State $(i,j) \gets R_{\ell}$ \Comment{Obtain the group pair index which contains heavy inner product vector pair}
    \State $p \gets \textsc{SolveOneGroupPair}(A_i, B_j, g)$  \Comment{Algorithm~\ref{alg:light_bulb} 
    } 
    \State $F.\textsc{Append}(p)$ 
\EndFor
\State \Return $F$

\EndProcedure

\end{algorithmic}
\end{algorithm*}
\fi

\begin{lemma}\label{lem:deterministic_algo}
For every $\rho>0$, there is a $c_0>0$ and a deterministic algorithm (\textsc{FindCorrelated} in \ifdefined\isshort Algorithm 3 in full version of the paper \cite{full} \else Algorithm~\ref{alg:light_bulb_deterministic} \fi) such that the heavy inner product between two sets problem in Definition~\ref{def:main_problem} of correlation $\rho$ can be solved in deterministic time 
$
O(n^{2 \omega / 3+o(1)})
$
on almost all instances whenever $d=c_0\log n$.
\end{lemma}
 \ifdefined\isshort

We defer the proofs into the full version of the paper \cite{full}.
 \else
 
\begin{proof}

In the randomized algorithm described in Algorithm~\ref{alg:light_bulb}, the only randomness used is 
the choice of independently and uniformly random $a^{x} \in\{-1,1\}$ for each $x \in A$ and the two pairwise independent hash functions, which requires $\Theta(n)$ random bits.

Because we repeat the entire algorithm $\Theta(\log n)$ times to get the correctness guarantee with high success probability, we need $\Theta(n \log n)$ random bits in total.


By standard constructions, we can use $O(\log n)$ independent random bits  to generate $n$ pairwise-independent random bits for the pairwise-independent $a^{x}$ variables. Therefore, we only needs $O(\log^2 n)$ independent random bits in the \textsc{FindCorrelated} algorithm. We can also use the bits to construct two pairwise independent hash functions used for partitioning the input sets into groups.

Our deterministic algorithm executes as follows:
\begin{itemize}
    \item Choose the same $c_0$ value as in Theorem~\ref{thm:main}.
    \item Let $A, B \subset\{-1,1\}^{d}$ be the input vectors. 
    Initialize an empty set $\wt{A} \subset\{-1,1\}^{d}$.
    \item For $i = 1 \rightarrow n$, brute-force test whether $a_i$
    is correlated with any vector in $B$, where $a_i$ is a vector in $A$. If not, we save $a_i$ into $\wt{A}$ until $|\wt{A}| = \Theta(\log n)$. We can assume that the vectors in $\wt{A}$ are all uniformly random vectors from $\{-1,1\}^{d}$, because they do not produce heavy inner product with any vector in $B$. This can be done in
    $O(|\wt{A}| \cdot|B| \cdot d)=O(n \log ^{2}(n))$
    time.
    \item we can use $\wt{A}$ as $d \cdot|\wt{A}|=\Theta(\log^{2} n)$ independent uniformly random bits. We thus use them as the required randomness to run the \textsc{FindCorrelated} on input vectors in $A$ and $B$. That algorithm has polynomially low error according to Theorem~\ref{thm:main}.
\end{itemize}

Because the time spent on checking if a subset $\wt{A}$ contains the heavy inner product pair is only $O(n \log^2(n))$, and the second part of Algorithm~\ref{alg:light_bulb_deterministic} takes $O(n^{2\omega/3 + o(1)})$ time according to Theorem~\ref{thm:main}, we have the overall time complexity:
\begin{align*}
   ~ O(n \log^2(n)) + O(n^{2\omega/3 + o(1)}) 
    =  ~ O(n^{2\omega/3 + o(1)})
\end{align*}
 
This completes the proof.

\end{proof}

\fi

\section{Speedup Neural Networks Training}
\label{sec:app_nn}



The training of a neural network usually consists of forward and backward propagation, which includes the computations as a function of the input, the network weights and the activation functions. In practice, not all neurons are activated so we can use shifted ReLU function \cite{syz21}:
$
    \sigma (x) = \max\{ \langle w_r, x\rangle, b_r  \}
    $
to speedup the training. Then it's possible that there is an algorithm whose running time outperforms the naive ``calculate and update" method, by detecting the activated neurons and update accordingly. This setting coincides with heavy inner product identification: We can view the forward propagation of the neural network as 1)\ identify the heavy inner product, where the goal is to identify the heavy inner product (between input and weight) which surpasses the activation function threshold and activates corresponding neurons. 2)\ forward the neurons with heavy inner product to the next layer, and update the desired function accordingly.

\begin{definition}[Two Layer Fully Connected Neural Network]\label{def:two_layer_NN}
We will consider two-layer fully connected neural networks with $m$ hidden neurons using shifted ReLU $\sigma : \R \rightarrow \R$ as the activation function.
We first define the two-layer shifted ReLU activation neural network
$
    f(x) := \sum_{r=1}^{m} a_r \cdot \sigma_{\tau}( \langle w_r, x \rangle )
$
where $\sigma_{\tau}(z) = \max\{z,\tau \}$ is the shifted ReLU activation function, $\{w_r\}_{r=1}^{m} \subset \R^{d}$ is the weight vector, $\{a_r\}_{r=1}^{m} \subset \R$ is the output weight vector, and $x \in \R^{d}$ is the input vector.
\end{definition}

  Consider activation threshold $\tau = \rho d$. In each iteration of gradient descent, it is necessary for us to perform prediction computations for each point in the neural network. For each input vector $x_i \in \R^{d}$, it needs to compute $m$ inner product in $d$ dimensions. Therefore,  $\Theta(mnd)$ time is a natural barrier to identify which weight and input index pairs surpass the shifted ReLU activation function threshold $\tau$ per training iteration.

When the neurons are sparsely activated, such a naive linear scan of the neurons is what we want to avoid. We want to locate the heavy inner product pairs that activate the neurons between the weight vectors $\{w_r\}_{r=1}^{m} \subseteq \R^{d}$ and input vectors $\{x_i\}_{i=1}^{n} \subseteq \R^d$. The methods we proposed in this paper can solve this problem efficiently. For the ReLU, it has been observed that the number of neurons activated by ReLU is small, so our method may potentially speed up optimization for those settings as well.

\section{Conclusion}\label{sec:conclusion}
In this paper, we design efficient algorithms to identify heavy inner product between two different sets. Given two sets $A \subset \{-1,+1\}^d$ and $B \subset \{-1,+1\}^d$ with $|A|=|B|=n$, we promise that we can find $k$ pairs of 
$(a, b) \in A \times B$ such that all of them satisfy $\langle a,b\rangle \geq \rho \cdot d$, for some constant $\rho$. 
Both the deterministic and the probabilistic algorithm run in $O(n^{2\omega /3 + o(1)}) = O(n^{1.582 + o(1)})$  
time to find all the heavy inner product pairs with high probability. In general, our algorithm will consume energy when running, but the theoretical guarantee helps speedup solving the heavy inner product identification problem. However, there are still some limitations of our works. For example, whether how we calculate the correlation between groups is the most efficient way to the heavy inner product problem.
As far as we know, our work does not have any negative societal impact.

\ifdefined\isarxiv
\bibliographystyle{alpha}
\bibliography{ref}
\else
\bibliographystyle{IEEEtran} 
\bibliography{ref}
\fi





\end{document}